\relax
\documentclass{article} 

\usepackage{times}  
\usepackage{helvet} 
\usepackage{courier}  
\usepackage[hyphens]{url}  
\usepackage{graphicx} 
\def\UrlFont{\rm}  
\usepackage{graphicx}  
\usepackage[ruled,linesnumbered,vlined]{algorithm2e}
\usepackage{mathtools, nccmath}
\usepackage{mathtools}
\usepackage{lipsum}

\usepackage[nonatbib,final]{neurips_2019}  

\usepackage{multicol}
\usepackage{lipsum}
\usepackage{tikz}
\usetikzlibrary{positioning,angles,quotes,arrows,automata}

\usepackage{float}
\usepackage{bm}
\usepackage{amsfonts}
\usepackage{amsbsy}
\usepackage{amsthm}
\usepackage{color}
\usepackage{amsmath}
\usepackage{enumerate}
\usepackage{amssymb}
\usepackage{enumitem}
\usepackage{array}

\usepackage{tabularx}
\usepackage{booktabs}
\usepackage{multirow}
\newcolumntype{Y}{>{\centering\arraybackslash}X}

\usepackage{color, colortbl}
\definecolor{Gray}{gray}{0.98}
\definecolor{LightCyan}{rgb}{0.88,1,1}
\newcolumntype{g}{>{\columncolor{Gray}}c}

\DeclarePairedDelimiter\abs{\lvert}{\rvert}%
\DeclarePairedDelimiter\norm{\lVert}{\rVert}%
\makeatletter
\let\oldabs\abs
\def\abs{\@ifstar{\oldabs}{\oldabs*}}
\let\oldnorm\norm
\def\norm{\@ifstar{\oldnorm}{\oldnorm*}}
\makeatother

\DeclarePairedDelimiter\ceil{\lceil}{\rceil}

\renewcommand{\*}[1]{{\pmb{#1}}}
\newcommand{\eat}[1]{}
\newcommand{\states}{\mathcal{S}}
\newcommand{\actions}{\mathcal{A}}

\newcommand*{\tr}{^{\mkern-1.5mu\mathsf{T}}}

\renewcommand{\P}{\mathbb{P}}

\newenvironment{mprog}{\begin{array}{>{\displaystyle}l>{\displaystyle}l>{\displaystyle}l}}{\end{array}}
\newcommand{\stc}{\\[1ex]  \mbox{s.t.} &}

\newcommand{\maximize}[1]{\max_{#1} &}
\newcommand{\one}{\mathbf{1}}
\newcommand{\Real}{\mathbb{R}}
\renewcommand{\ss}{\,:\,}

\newcommand{\zeros}{\mathbf{0}}
\newcommand{\linf}{L_\infty}
\newcommand{\lw}{_{1,\*w}}
\newcommand{\liw}{_{\infty,\*w}}

\newcommand{\real}{\mathbb{R}}
\newcommand{\prob}{\mathbb{P}}
\newcommand{\simplexs}{\Delta ^{ \states }}

\newcommand{\ambset}{\mathcal{P}}
\newcommand{\expect}{\mathbb{E}}
\newcommand{\sa}{_{s,a}}

\newcommand{\qeu}{\mathcal{Q}}

\newcommand{\dataset}{\mathcal{D}}
\newcommand{\opt}{^\star}

\newcommand{\citeasnoun}[1]{\citeauthor{#1} (\citeyear{#1})}

\usepackage[numbers,sort&compress]{natbib}

\newtheorem{theorem}{Theorem}[section]
\newtheorem{corollary}{Corollary}[section]
\newtheorem{lemma}[theorem]{Lemma}

\DeclareMathOperator{\Med}{Med}

\newcommand{\bahram}[1]{\textcolor{red}{[#1]}}

\usepackage[noabbrev,capitalize]{cleveref}

\title{Optimizing Norm-Bounded Weighted Ambiguity Sets for Robust MDPs}

\author{
	Reazul Hasan Russel \thanks{Equal contribution} \hspace{10pt} Bahram Behzadian \footnotemark[1] \hspace{10pt} Marek Petrik\\ 
	Department of Computer Science\\
	University of New Hampshire\\
	Durham, NH 03824 USA\\
	{\tt rrussel}, {\tt bahram}, {\tt mpetrik} @ {\tt cs.unh.edu}
}

\begin{document}

\maketitle

\begin{abstract}
  Optimal policies in Markov decision processes (MDPs) are very sensitive to model misspecification. This raises serious concerns about deploying them in high-stake domains. Robust MDPs (RMDP) provide a promising framework to mitigate vulnerabilities by computing policies with worst-case guarantees in reinforcement learning. The solution quality of an RMDP depends on the ambiguity set, which is a quantification of model uncertainties. In this paper, we propose a new approach for optimizing the shape of the ambiguity sets for RMDPs. Our method departs from the conventional idea of constructing a norm-bounded uniform and symmetric ambiguity set. We instead argue that the structure of a \emph{near-optimal} ambiguity set is problem specific. Our proposed method computes a weight parameter from the value functions, and these weights then drive the shape of the ambiguity sets. Our theoretical analysis demonstrates the rationale of the proposed idea. We apply our method to several different problem domains, and the empirical results further furnish the practical promise of weighted near-optimal ambiguity sets.
\end{abstract}

\section{Introduction} \label{sec:introduction}
Markov decision processes (MDPs) provide a framework for representing dynamic decision-making problems under uncertainty~\cite{Bertsekas1996,puterman2005,sutton2018reinforcement}. An MDP model assumes that the exact transition probabilities and rewards are available. However, for the most realistic control problems, the underlying MDP model is not known precisely. While one may have full access to state space and actions, the transition probabilities are rarely known with confidence and must be instead estimated from data. Even small transition errors can significantly degrade the quality of the optimal policy~\cite{Wiesemann2013}. This work focuses primarily on the \emph{reinforcement learning} setting in which transition probabilities are estimated from samples, and the errors are due to having a small sample.

Robust MDPs~(RMDPs) are a convenient model for computing policies that are insensitive to small errors in transition probabilities ~\cite{Nilim2004,Iyengar2005,Wiesemann2013}. The basic idea of RMDPs is to compute the best policy for the worst-case realization of transition probabilities. The model can be seen as a zero-sum game against an adversarial nature. The decision-maker chooses the best action, and nature chooses the worst-case transition probability. The set of possible transition probabilities that nature can choose from is known as the \emph{ambiguity set} or the uncertainty set.

The main challenge in using RMDPs is computing solutions that are robust without being overly conservative~\cite{petrik2019beyond, Russel2018, Tirinzoni2018, petrik2016safe}. The trade-off between the robustness and average-case performance is determined primarily by choice of the ambiguity set. The typical optimization problem solved by the adversarial nature in RMDPs is as follows:
\[ \min_{\*p \in \Delta^S} \left\{ \*p\tr \*v ~:~ \norm{\*p - \bar{\*p}}_1 \le \psi \right\}, \]
where $\bar{\*p}$ is the expected (or nominal) transition probability, $\Delta^S$ is the probability simplex over $S$ states, and $\psi$ is the size of the ambiguity set. That is, the ambiguity set is defined in terms of the $L_1$ distance from the nominal solution. A large ambiguity set, of course, leads to more conservative solutions~\cite{gupta2019near}, but the \emph{shape} of the set often plays an even more important role.

As the main contribution, we develop 1) a new technique for understanding the impacts of the ambiguity set choice on solution quality and 2) an algorithm that can optimize the shape of the ambiguity set for a particular problem. Our results make it possible to answer questions like: ``Should I use $L_1$ or $L_\infty$ norm to define my ambiguity set?'', or ``Can I get better results when I use a weighted $L_1$ norm?'' We show that the set shape is primarily driven by the structure of the value function. For example, an $L_\infty$ set is likely to work better than the $L_1$ set when the value function is sparse. As a secondary contribution, we also establish new finite-sample guarantees for transition probabilities with $L_\infty$, weighted $L_1$, and weighted $L_\infty$ norms.

\section{Framework}\label{sec:framework}

Our overall goal is to solve an MDP that is known only approximately. This is relevant, for example, in model-based reinforcement learning when the MDP is estimated from an incomplete dataset. The MDP has a finite number of states $\states = \{1, \ldots, S \}$ and actions $\actions = \{1, \ldots, A\}$. The decision-maker can take any action $a \in \actions$ in every state $s \in \states$ and receives a reward $r_{s,a} \in \real$ . The action results in a transition to the next state $s'$ according to the transition probabilities $\*p_{s,a}\opt \in \simplexs$. We use $P\opt : \states\times\actions\to\simplexs$ to denote the transition kernel and $\*p_{s,a}$ to denote the vector of transition probabilities from state $s$ and action $a$. Note that the value $P\opt$ represents the true transition probability which may be unknown. 

The objective in solving the MDP is to compute a policy $\pi: \states \rightarrow \actions$ that maximizes the infinite-horizon $\gamma$-discounted return $\rho$. The discounted return for a policy $\pi$ and a given transition kernel $P$ is defined as follows:
$\rho(\pi,P) = \expect \left[ \sum_{t=0}^{\infty} \gamma ^t \cdot r_{S_t, \pi(S_t)} \right]$. Ideally, the optimal policy $\pi\opt$ could be computed to maximize the true discounted return $\pi\opt \in \arg \max_{\pi\in\Pi} \rho(\pi, P\opt)$, where $\Pi$ is the set of all policies. This is often impossible, since the true transition probabilities $P\opt$ are, unfortunately, rarely known with precision.

Robust MDPs address the challenge of unknown $P\opt$ by considering a broader set of possible transition probabilities. Instead of computing the best policy for a specific transition kernel $P$, the goal is to compute the best policy for a range of kernels $\ambset$. In other words, the objective is to compute a policy that is best with respect to the worst-case choice of the transition probabilities:
\begin{equation}\label{eq:rmdp}
\max_{\pi \in \Pi_R} \min_{P \in \ambset} \rho (\pi , P)~.
\end{equation}
Because solving the general problem in \eqref{eq:rmdp} is NP-hard~\cite{Nilim2004,Iyengar2005}, most research has focused on so-called $(s,a)$-rectangular ambiguity sets $\ambset$~\cite{Wiesemann2013,LeTallec2007}. We use $\ambset_{s,a} \subseteq \simplexs$ to denote the ambiguity set for a state $s$ and an action $a$. The optimal robust value function $\hat{\*v}\opt \in \real^S$ in $(s,a)$-rectangular RMDPs must satisfy the robust Bellman optimality condition:
\begin{equation}  \label{eq:robust_update}
\hat{\*v}\opt(s) = \max_{a \in \actions}\min_{\*p \in \ambset_{s,a}} r_{s,a} + \gamma \, \*p\tr \hat{\*v}\opt ~.
\end{equation}
The ambiguity set $\ambset_{s,a}$ is typically defined as:
\[ \ambset_{s,a} = \left\{ \*p \in \simplexs ~:~ \norm{\*p - \bar{\*p}\sa }_1 \le \psi_{s,a} \right\}, \]
where $\bar{\*p}\sa$ is the nominal transition probability that is estimated from data. The size $\psi_{s,a}$ determines the level of robustness: a larger $\psi_{s,a}$ leads to more robust solutions.

When facing limited sample availability, the size  $\psi_{s,a}$ is usually chosen such that $\*p\opt$ is contained in the ambiguity set with probability  $1-\delta$: 
\[\prob \left[ \*p\opt\sa \in \ambset\sa \right] \ge 1 - \delta ~,\]
where $\delta\in (0,1]$ is the confidence level. Using standard frequentist bounds, this requirement translates to~\cite{petrik2016safe,thomas2015high,Weissman2003xx,petrik2019beyond}: \[
\psi_{s,a} = \sqrt{\frac{2}{n\sa} \log \frac{SA2^S}{\delta}} ~,
\]
where $n\sa$ is the number of transitions from state $s$ by taking action $a$ in $\dataset$. One important benefit of using ambiguity sets of this type is that the solution of the RMDP provides a guarantee on the return of the MDP with confidence $1-\delta$.

\paragraph{Research Objective.} The goal of this work is to design ambiguity sets that provide the highest possible guaranteed return for a given confidence level of $1-\delta$. This problem can be loosely formalized for each $s$ and $a$ as follows:
\begin{equation} \label{eq:main_objective}
\begin{mprog}
\maximize{\ambset\sa} \min_{\*p \in \ambset_{s,a}} r_{s,a} + \gamma \, \*p\tr \hat{\*v}\opt
\stc \prob \left[ \*p\opt\sa \in \ambset\sa \right] \ge 1 - \delta~.
\end{mprog}
\end{equation}
Note that, since the Bellman operator is monotone, maximizing the value of each state individually maximizes the entire value function. The distributionally-constrained optimization problem in \eqref{eq:main_objective} is, of course, intractable~\cite{Ben-Tal2009}. As stated, it also relies on knowing the optimal robust value function $\hat{\*v}\opt$, which itself depends on the choice of $\ambset$. We instead examine a version of \eqref{eq:main_objective} restricted to optimizing the weights of an $L_p$ norm and assume that a rough estimate of $\hat{\*v}\opt$ is available.

\section{Value-Function Driven Ambiguity Sets} \label{sec:shape}
In this section, we outline the general approach to tackling the desired optimization in \eqref{eq:main_objective}. We relax the problem and use strong duality theory to get bounds that can be optimized tractably. Since this section is restricted to a single state and action, we drop the state and action subscript throughout.

The general approach to the construction of a good ambiguity set will rely on relaxing the robust optimization problem. This relaxation makes it possible to get an analytical expression for the robust problem and use it to guide the selection between different sets $\ambset$. In the remainder of the section, we use $\*z$ to denote a given estimate of the optimal robust value function. Recall the robust Bellman update \eqref{eq:robust_update} can be simplified as follows:
\[
q(\*z) = \min_{\*p \in \simplexs} \left\{ {\*p}\tr \*z : \norm{ \*p - \bar{\*p} } \leq \psi \right\}~,
\]
since $r_{s,a}$ and $\gamma$ are constants independent of $\*p$. The value of $q(\*z)$ represents the expected value of the next state. Notice that the optimization is stated in terms of a generic norm.

We can now derive a \emph{lower} bound on the value $q$. We later choose the shape of the ambiguity set to maximize this lower bound. By relaxing the non-negativity constraints on $\*p$, we get the following optimization problem:
\[ q(\*z) \ge \min_{\*p \in \real^S} \left\{ {\*p}\tr \*z : \norm{ \*p - \bar{\*p} } \leq \psi,\; \one\tr \*p = 1 \right\}~.
 \]
Here, $\one$ is a vector of all ones of the appropriate size. Dualizing this optimization problem and following algebraic manipulation, we get the reformulation described in the following theorem.

\begin{theorem}\label{thm:choose_weights}
	The estimate of expected next value can be lower bounded as follows:
	\begin{equation} \label{eq:lower_bound_dual}
	q(\*z) \ge \bar{\*p}\tr  \*z -  \min_\lambda  ~~  \psi \norm{ \*z + \lambda \*1}_\star~.
	\end{equation}
\end{theorem}

The result in \cref{thm:choose_weights} relies on the \emph{dual norm}, which is defined as:
\[\norm{\*z}_{\star}=\sup\{\*z^{\intercal } \*x\;|\; \norm{\*x} \le 1 \}~.\]	
It is well known that dual norms to $L_1, L_2$, and  $L_\infty$ are norms $L_\infty, L_2$, and $ L_1$ respectively.

The lower bound in \eqref{eq:lower_bound_dual} is still not quite analytical as it involves solving an optimization problem. This is, however, a single-dimensional optimization, and we show that it does have an analytical form for common norm choices. In the remainder of the section, we derive the specific form of \cref{thm:choose_weights} for weighted $L_1$ and $L_\infty$ norms. We also describe algorithms that optimize the weights in order to maximize the expected robust value.

We generalize the results also to weighted $p$-norms, which are usually defined as follows. The weighted $L_1$ and $L_\infty$ norms for a set of \emph{positive} weights $\*w \in \real_+^S$ and $\*w > \zeros$ are defined as:
\[ \norm{\*z}_{1,\*w}  = \sum_{i=1}^S w_i \lvert z_i \rvert~, \quad  \norm{\*z}_{\infty,\*w}  = \max_{i=1,\ldots,S} w_i \lvert z_i \rvert~. \]

Using this fact, \cref{thm:choose_weights} can be specialized to $L_1$  weighted ambiguity sets as follows.
\begin{corollary}[Weighted $L_1$ Ambiguity Set] 
	\label{cor:ambiguity_l1_bound}
	Suppose that $q(\*z)$ is defined in terms of a weighted $L_\infty$ norm for some $\*w > \zeros$:
	\[
	q(\*z) = \min_{\*p \in \simplexs} \left\{ {\*p}\tr \*z : \norm{ \*p - \bar{\*p} }_{1,\*w} \leq \psi \right\}~.
	\]
	Then $q(\*z)$ can be lower-bounded as follows:
	\[ q(\*z) \ge \bar{\*p}\tr  \*z - \psi \norm{ \*z - \lambda \*1}_{\infty, \frac{1}{\*w}} ~, \]
	for any $\lambda$. Moreover, when $\*w = \one$, the bound is tightest when $\lambda = (\max_i z_i + \min_i z_i) / 2$ and the bound turns to $q(\*z) \ge \bar{\*p}\tr  \*z -  \frac{\psi}{2} \norm{ \*z }_{s} $ with $\norm{\cdot}_s$ representing the \emph{span semi-norm}.	
\end{corollary}

Since the dual norm of a dual norm is the original norm, we also get a similar result for weighted $L_\infty$ ambiguity sets.
\begin{corollary}[Weighted $L_\infty$ Ambiguity Set]
	\label{cor:ambiguity_l8_bound}
	Suppose that $q(\*z)$ is defined in terms of a weighted $L_\infty$ norm for some $\*w > \zeros$:
	\[
	q(\*z) = \min_{\*p \in \simplexs} \left\{ {\*p}\tr \*z : \norm{ \*p - \bar{\*p} }_{\infty,\*w} \leq \psi \right\}~.
	\]
	Then $q(\*z)$ can be lower-bounded as follows:
	\[ q(\*z) \ge \bar{\*p}\tr  \*z - \psi \norm{ \*z - \lambda \*1}_{1, \frac{1}{\*w}} ~, \]
	for any $\lambda$. Moreover, when $\*w = \one$, the bound is tightest when $\lambda$ is the \emph{median} of $\*z$.
\end{corollary}
	The optimal $\lambda$ being a median follows because maximization over $\lambda$ values is identical to the formulation of the optimization problem for the \emph{quantile regression}. 
	
The utility of \cref{cor:ambiguity_l1_bound,cor:ambiguity_l8_bound} is twofold: 1) we will use it to decide whether $L_1$ or $L_\infty$ ambiguity sets are more appropriate for a given problem, and 2) we will use them to improve solution quality by optimizing the weights involved. 

\subsection{Optimizing Norm Weights}

In this section, we introduce methods for optimizing weights that provide the tightest possible guarantees. To simplify the exposition, we first assume weighted $L_1$ ambiguity sets and then describe a similar approach for the $L_\infty$ ambiguity sets.

Recall that the objective is to choose an ambiguity set that leads to a solution with the maximal objective value that simultaneously provides the required performance guarantees:
\begin{equation} \label{eq:weight_objective}
\begin{mprog}
\maximize{\*w \in \real^S_{++}} \min_{\*p \in \Delta^S} \left\{ {\*p}\tr \*z ~:~ \norm{ \*p - \bar{\*p} }_{1,\*w} \leq \psi \right\} 
\stc \sum_{i=1}^S w_i^2 = 1 
\end{mprog}~.
\end{equation}
The purpose of the constraint $\sum_{i=1}^S w_i^2 = 1$ is to normalize $\*w$ to preserve the desired robustness guarantees with \emph{the same} $\psi$. Notice that scaling both $\*w$ and $\psi$ simultaneously does not change the ambiguity set. The justification for this particular choice of the regularization constraint is given formally in \cref{sec:size}. To summarize, this constraint makes it possible to treat $\psi$ as being independent of $\*w$.

Because the optimization problem in \eqref{eq:weight_objective} is intractable (a non-convex optimization problem), we instead maximize a lower bound on the objective established in \cref{cor:ambiguity_l1_bound}:
\begin{equation} \label{eq:weight_objective_ref}
\max_{\*w \in \real^S_{++}} \left\{ \bar{\*p}\tr  \*z - \psi \norm{ \*z - \bar\lambda \*1}_{\infty, \frac{1}{\*w}} ~:~ \sum_{i=1}^S w_i^2 = 1 \right\} ~.
\end{equation}
The value $\bar{\lambda}$ is fixed ahead of time and does not change with a different choice of the weights $\*w$. Omitting terms that are constant with respect to $\*w$ gives the following formulation for the (approximately) optimal choice of weights $\*w$:
\begin{equation} \label{eq:obj_reformulated}
\*w\opt \in \arg\min_{\*w \in \real^S_{++}} \left\{ \norm{ \*z - \bar\lambda \*1}_{\infty, \frac{1}{\*w}} ~:~ \sum_{i=1}^S w_i^2 = 1 \right\} ~. 
\end{equation}

The nonlinear optimization problem in \eqref{eq:obj_reformulated} is convex and can be, surprisingly, solved \emph{analytically}.  To simplify notation, let $b_i = \abs{z_i - \bar{\lambda}}$ for $i=1,\ldots, S$. After introducing an auxiliary variable $t$, the optimization problem becomes:
\begin{equation} \label{eq:compute_weight_analytically_l1}
\min_{\*w > \zeros,t}  \left\{ t ~:~ t \ge b_i/w_i,\, \sum_{i=1}^S w_i^2 = 1 \right\}~.
\end{equation}
The constraints $\*w > \zeros$ cannot be active (because of $1/w_i$) and may be safely ignored. That means the convex optimization problem in \cref{eq:compute_weight_analytically_l1} has a linear objective and  $S+1$ variables ($\*w$'s and $t$) and $S+1$ constraints. All the constraints, therefore, must be active in the optimal solution~\cite{Bertsekas2003}. The optimal $\*w\opt$ thus satisfies:
\begin{equation} \label{eq:analytical_weights_l1}
w_i\opt = \frac{b_i}{\sqrt{\sum_{j=1}^S b_j^2} } ~.
\end{equation}
Since $\sum_i w_i^2 = 1$ implies $\sum_i b_i^2 / t^2 =1$, we can conclude $ t= \sqrt{ \sum_i b_i}$.

Following the same approach for the weighted $L_\infty$ ambiguity set, the equivalent optimization of \eqref{eq:compute_weight_analytically_l1} becomes:
\begin{equation} \label{eq:compute_weight_analytically_l8}
\min_{\*w > \zeros} \; \left\{ \sum_{i=1}^S b_i / w_i ~:~ \sum_{i = 1}^S w_i^2 = 1 \right\}~.
\end{equation}
Again, the non-negativity constraints on $\*w$  can be relaxed. Using the necessary optimality conditions (and a Lagrange multiplier), the optimal weights $\*w$ are:
\begin{equation} \label{eq:analytical_weights_l8}
w_i\opt = \frac{b_i^{1/3}}{\sqrt{\sum_{j=1}^S {b_j^{2/3}}} } ~.
\end{equation}

In this section, we described a new approach for optimizing the shape of ambiguity sets. In the following section, we establish new sampling bounds for these new types of ambiguity sets.

\section{Size of Ambiguity Sets} \label{sec:size}

In this section, we describe new sampling bounds that can be used to construct ambiguity sets with desired sampling guarantees. We describe both frequentist and Bayesian methods.

\begin{algorithm}
	\KwIn{Distribution $\theta$ over $\*p\opt_{s,a}$, confidence level $\delta$, sample count $m$, weights $w$}
	\KwOut{Nominal point $\bar{\*p}_{s,a}$ and $\psi_{s,a}$}
	Sample $X_1, \ldots, X_m \in \Delta^S$ from $\theta$: $X_i \sim \theta $\;
	Nominal point: $\bar{\*p}_{s,a} \gets (1/ m) \sum_{i=1}^m X_i $\;
	Compute distances $d_i \gets \lVert \bar{\*p}_{s,a} - X_i \rVert_{1,w}$ and sort in \emph{increasing} order\;
	$\psi_{s,a} \gets d_{\ceil{(1-\delta)m}}$\;
	\Return{$\bar{\*p}_{s,a}$ and $\psi_{s,a}$}\
	\caption{Weighted Bayesian Credible Interval (WBCI)} \label{alg:bayes}
\end{algorithm}

\paragraph{Bayesian Credible Regions (BCR).}
In Bayesian statistics, credible intervals are comparable to classical confidence intervals. Credible intervals are fixed bounds on the estimator, which itself is a random variable. The Bayesian approach combines the prior domain knowledge with observations to infer current belief in the form of the posterior distribution of the estimator~\cite{bertsekas2002introduction}. 
\citeasnoun{petrik2019beyond} suggest an approach to construct ambiguity regions from credible intervals.   The method starts with sampling from the posterior probability distribution of $P\opt$ given data $\dataset$ to estimate the mean transition probability $\bar{\*p}\sa = \expect_{P\opt} [\*p\opt\sa | \dataset]$. Then  the smallest possible ambiguity set around the mean is obtained by solving the following optimization problem for each state $s$ and action $a$: 
\begin{equation*} \label{eq:optimization_bci}
\psi\sa^B = \min_{\psi\in\Real_{++}} \left\{\psi \ss \P\left[ \norm{\*p\opt_{s,a} - \bar{\*p}_{s,a}} > \psi ~:~ \mathcal{D} \right] < \frac{\delta}{SA} \right\}~.
\end{equation*}
Finally, the Bayesian ambiguity set can be obtained by:
\begin{equation*}
\label{eq:bay_ambset}
\ambset\sa^B = \left\lbrace \*p \in \simplexs : \norm{\*p - \bar{\*p}\sa} \leq \psi\sa^B \right\rbrace~.
\end{equation*}
This construction applies easily to any form of norm used in the construction of ambiguity sets. That is, it is easy to generalize this method for both weighted $L_1$ and weighted $L_\infty$ ambiguity sets that we study in this work. \cref{alg:bayes} summarizes the simple algorithm to construct weighted Bayesian ambiguity sets.

\subsection{Weighted Frequentist Confidence Intervals (WFCI)}
Confidence intervals obtained by Hoeffding's inequality are based on the empirical mean of independent, bounded random variables. In this section, we introduce confidence regions with weighted $L_1$ bound on transition probabilities as an extension to Lemma~(C.1) presented by~\citeasnoun{petrik2019beyond}.

\begin{theorem}\label{thm:weighted_lone}
	Suppose that $\bar{\*p}\sa$ is the empirical estimate of the transition probability obtained from $n\sa$ samples for some $s \in \states$ and $ a \in \actions$. If the weights $\*w \in \real_{++}^S$ are sorted in non-increasing order $w_i \ge w_{i+1}$, then:
	\begin{equation*} \label{hoeff_weighted}	
	\prob \left[ \norm{\bar{\*p}\sa - \*p\opt\sa}_{1,\*w}  \geq \psi\sa  \right] \leq  2
	 \sum_{i = 1}^{S-1} 2^{S - i} \exp  \left(  -  \frac{\psi\sa^2n\sa}{2 w_i^2} \right)~.
	\end{equation*}
	Note that $\bar{\*p}$ is the random variable in the inequality above.
\end{theorem}
\begin{theorem}[weighted $\linf$ error bound] \label{thm:weighted_linfty}
	Suppose that $\bar{\*p}\sa$ is the empirical estimate of the transition probability obtained from $n\sa$ samples for some $s \in \states$ and $ a \in \actions$. Then:
	\begin{equation*}		
	\prob \left[ \norm{\bar{\*p}\sa - \*p^\star\sa}\liw \geq \psi\sa \right] \leq 2 \sum_{i=1}^S \exp \left(-2 \frac{\psi\sa^2 n\sa}{w_{i}^2} \right)~.
	\end{equation*}
\end{theorem}

\Cref{thm:weighted_lone,thm:weighted_linfty} establish the error bounds that can be used to construct ambiguity sets of appropriate size. Unlike with the standard error bound, $\psi_{s,a}$ cannot be determined readily from the bounds analytically. However, since the confidence level function is monotonically increasing, $\psi_{s,a}$ can be determined easily using a bisection method. 

\section{Empirical Evaluation}\label{sec:experiments}
In this section, we empirically evaluate the advantage of using weighted ambiguity sets in Bayesian and frequentist settings. We evaluate $L_1$ and $L_\infty$-bounded ambiguity sets, both with weights and without weights. We compare BCI with Hoeffding and Bernstein sets. We start by assuming a true underlying model that produces the simulated datasets containing $100$ samples for each state and action. The frequentist methods use these datasets to construct an ambiguity set. Bayesian methods combine the data with a prior to computing a posterior distribution and then draw $10000$ samples from the posterior distribution to construct a Bayesian ambiguity set. We use an uninformative uniform prior over the reachable next states for all the experiments unless otherwise specified. This prior is somewhat informative in the sense that it contains the knowledge of non-zero transitions implied by the datasets. The performance of the methods is evaluated by the guaranteed robust returns computed for a range of different confidence levels. We strengthen the weighted $L_1$ error
bound by a factor of two to match with the unweighted one.

\paragraph{Single Bellman Update}
In this experiment, we set up a very trivial problem to meticulously examine our proposed method. We consider a transition from a single state $s_0$ and an action $a_0$ leading to $5$ terminal states $s_1,\ldots,s_5$. The value functions are assumed to be fixed and known. The prior is uniform Dirichlet over the next states. Plots in \cref{fig:single_update} and \cref{fig:single_update_sparse} show a comparison of average guaranteed returns for $100$ independent trials. The weighted methods outperform unweighted methods in all instances. Also, the weighted BCI methods are significantly better than other frequentist methods. It is also apparent from the plot that the $\linf$-constrained method can outperform in case of sparse value functions as shown in \cref{fig:single_update_sparse}.

\begin{figure}
	\begin{minipage}{0.45\linewidth}
		\centering
		\includegraphics[width=\textwidth]{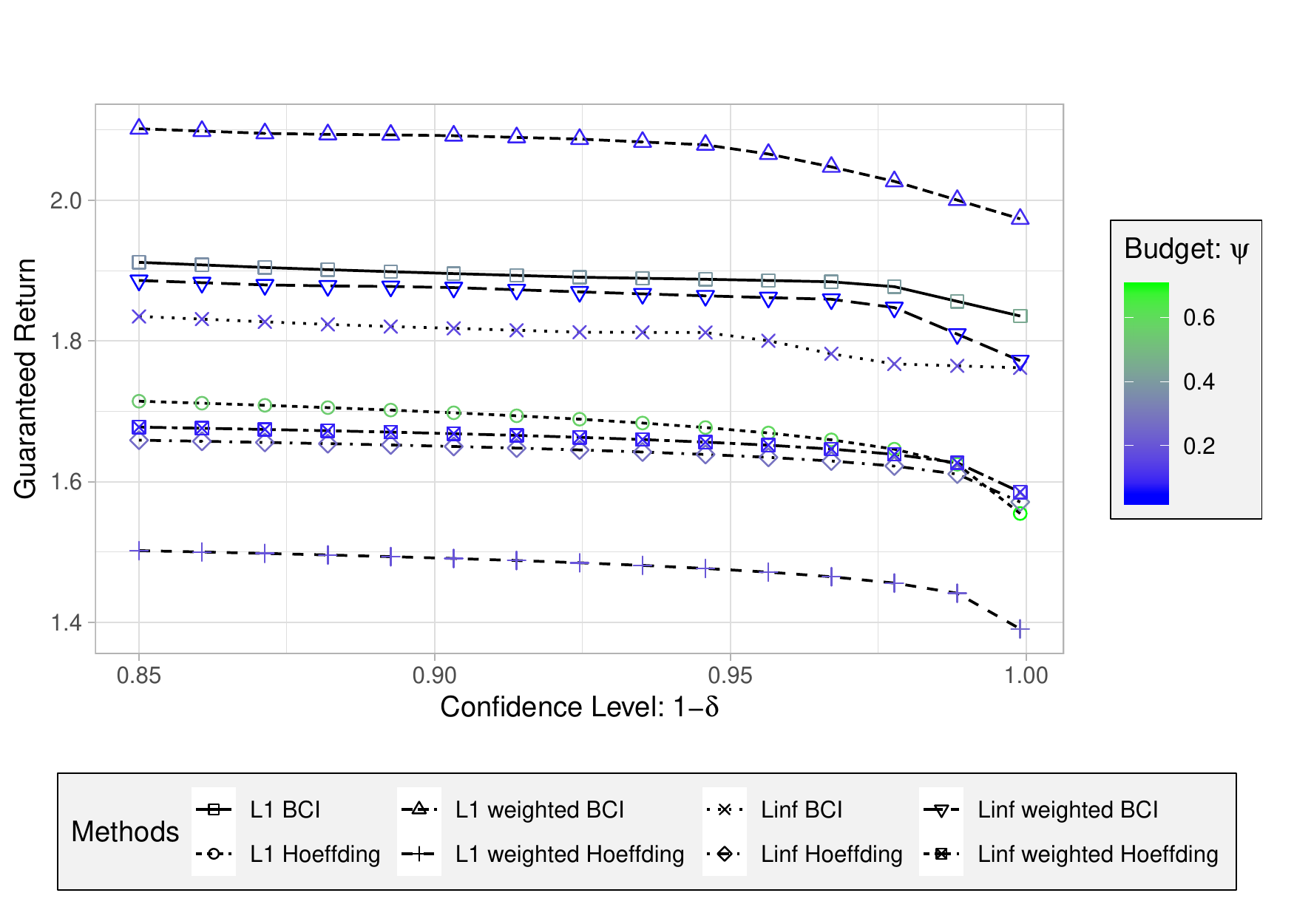}
\caption{Single Bellman Update: the guaranteed return for a monotonic value function $v = [1, 2, 3, 4, 5]$.}
\label{fig:single_update}
	\end{minipage}
	\hspace{0.5in}
	\begin{minipage}{0.45\linewidth}
		\centering
		\includegraphics[width=\textwidth]{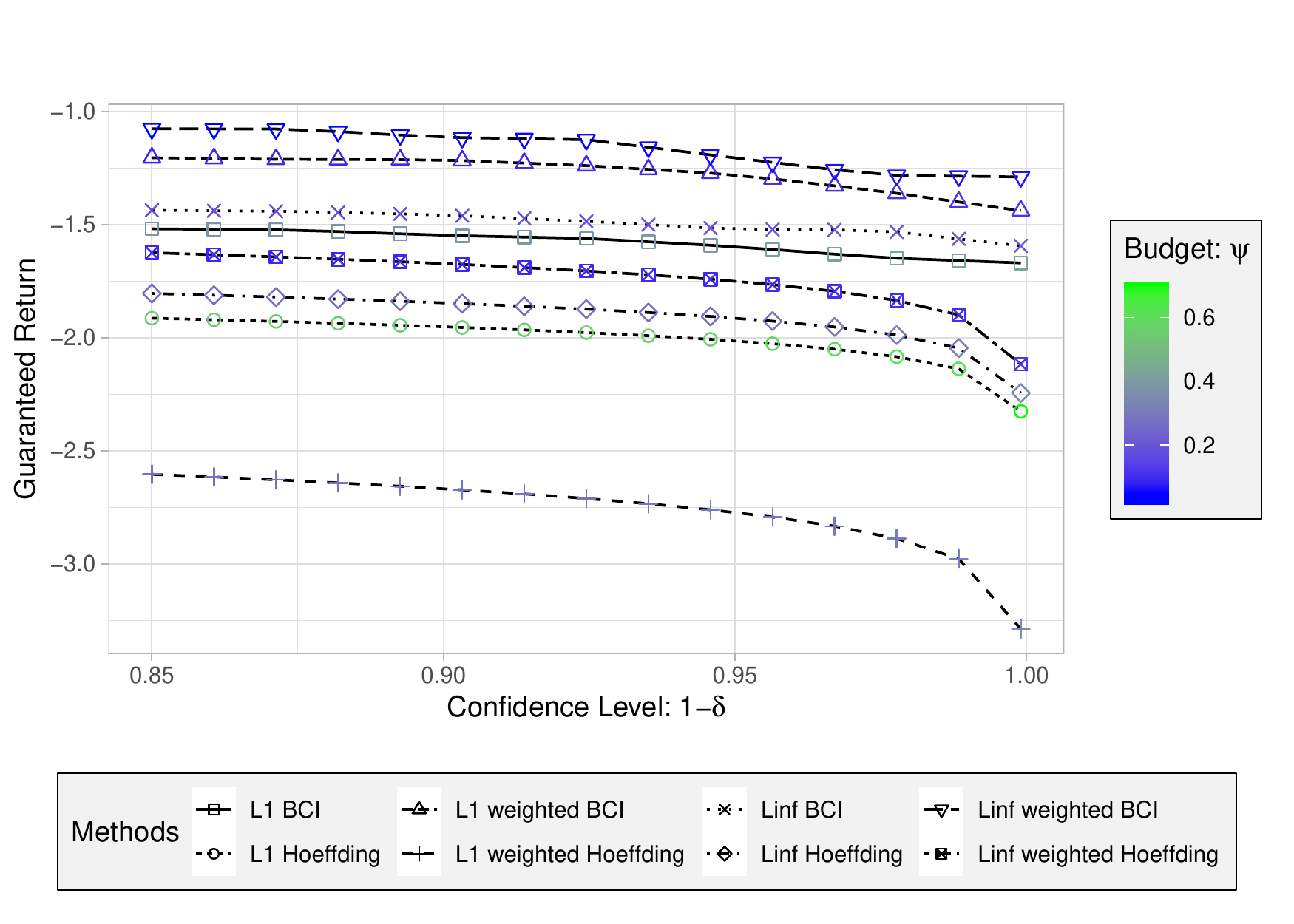}
		\caption{Single Bellman Update: the guaranteed return for a sparse value function $v = [0, 0, 0, 0, -5]$.}
		\label{fig:single_update_sparse}
	\end{minipage}
\end{figure}

\begin{table*}
	\centering
	\begin{tabularx}{\textwidth}{r *5{>{\raggedleft\arraybackslash}X}}
		\toprule
		& \multicolumn{1}{c}{Confidence $\rightarrow$} & \multicolumn{2}{c}{0.5} & \multicolumn{2}{c}{0.95} \\
		\cmidrule(lr){3-4} \cmidrule(lr){5-6} 
		& \multicolumn{1}{c}{Methods $\downarrow$} & \multicolumn{1}{c}{Unweighted} & \multicolumn{1}{c}{Weighted} & 
		\multicolumn{1}{c}{Unweighted} & 
		\multicolumn{1}{c}{Weighted} \\
		\midrule
		\multirow{ 2}{*}{Bayesian} 
		& $L_1$ BCI & 8198.32 & \textbf{30014.58} & 2278.14 & \textbf{21591.15} \\
		& $L_{\infty}$ BCI & 7999.96 & 26653.51 & 2210.42  & 17943.25 \\ \midrule
		\multirow{ 4}{*}{\centering Frequentist} 
		& $L_1$ Hoeffding & 497.66 & 1392.49 & 490.18 & 655.29 \\
		&$L_1$ Bernstein & 490.18 & 721.07 & 490.18 & 490.18 \\
		&$L_{\infty}$Hoeffding & 805.53 & \textbf{12513.47} & 490.18  & \textbf{7155.85} \\
		\bottomrule
	\end{tabularx}
	\caption{RiverSwim experiment. Guaranteed robust return for different confidence levels.} \label{tab:riverswim}
\end{table*}

\begin{table*}
	\centering
	\begin{tabularx}{\textwidth}{r *5{>{\raggedleft\arraybackslash}X}}
		\toprule
		& \multicolumn{1}{c}{Confidence $\rightarrow$} & \multicolumn{2}{c}{0.5} & \multicolumn{2}{c}{0.95} \\
		\cmidrule(lr){3-4} \cmidrule(lr){5-6} 
		& \multicolumn{1}{c}{Methods $\downarrow$} & \multicolumn{1}{c}{Unweighted} & \multicolumn{1}{c}{Weighted} & 
		\multicolumn{1}{c}{Unweighted} & 
		\multicolumn{1}{c}{Weighted} \\
		\midrule
		\multirow{ 2}{*}{Bayesian} 
		& $L_1$ BCI & -99973 & \textbf{-5675} & -107348 & \textbf{-7552} \\
		& $L_{\infty}$ BCI & -132111 & -32794 & -136121  & -46041 \\ 
		\midrule
		\multirow{ 4}{*}{\centering Frequentist}
		& $L_1$ Hoeffding & -106966 & -84615 & -110656 & -89607 \\
		&$L_1$ Bernstein & -131594 & -123646 & -132834 & -125979 \\
		&$L_{\infty}$Hoeffding & -132226 & \textbf{-28267} & -133427 & \textbf{-42236} \\
		\bottomrule
	\end{tabularx}
	\caption{Population experiment. Guaranteed robust return for different confidence levels.} \label{tab:population}
\end{table*}

\begin{table*}
	\centering
	\begin{tabularx}{\textwidth}{r *5{>{\raggedleft\arraybackslash}X}}
		\toprule
		& \multicolumn{1}{c}{Confidence $\rightarrow$} & \multicolumn{2}{c}{0.5} & \multicolumn{2}{c}{0.95} \\
		\cmidrule(lr){3-4} \cmidrule(lr){5-6} 
		& \multicolumn{1}{c}{Methods $\downarrow$} & \multicolumn{1}{c}{Unweighted} & \multicolumn{1}{c}{Weighted} & 
		\multicolumn{1}{c}{Unweighted} & 
		\multicolumn{1}{c}{Weighted} \\
		\midrule
		\multirow{ 2}{*}{Bayesian} 
		& $L_1$ BCI & 314.31 & \textbf{433.02} & 294.99 & \textbf{418.68} \\
		& $L_{\infty}$ BCI & 180.96 & 272.34 & 158.33  & 250.78 \\ 
		\midrule
		\multirow{ 4}{*}{\centering Frequentist}
		& $L_1$ Hoeffding & 195.11 & 240.74 & 184.02 & 233.36 \\
		&$L_1$ Bernstein & 124.30 & 196.95 & 109.71 & 185.95 \\
		&$L_{\infty}$Hoeffding & 138.57 & \textbf{252.96} & 124.09 & \textbf{242.16} \\
		\bottomrule
	\end{tabularx}
	\caption{Inventory experiment. Guaranteed robust return for different confidence levels.} \label{tab:inventory}
\end{table*}

\paragraph{RiverSwim}
We consider the standard RiverSwim~\cite{strehl2008analysis} domain for evaluating our methods. The process follows by sampling synthetic datasets from the true model and then computing the guaranteed robust returns for different methods. We use a uniform Dirichlet distribution over the next states as prior. \cref{tab:riverswim} summarizes the results. All the weighted methods dominate unweighted methods, and the weighted $L_1$ BCI method provides the highest guaranteed return.

\paragraph{Population Growth Model}
We also apply our method in an exponential population growth model~\citep{Kery2012}. Our model constitutes a simple state-space with exponential dynamics. At each time step, the land manager has to decide whether to apply a control measure to reduce the growth rate of the species. We refer to ~\citeasnoun{Tirinzoni2018} for more details of the model. The results are summarized in \cref{tab:population}. Returns for all the methods are negative, which implies a high management cost. Policies computed with frequentist and unweighted methods yield a very high cost. Bayesian and weighted methods significantly outperform other methods.

\paragraph{Inventory Management Problem}
Next, we take the classic inventory management problem~\citep{Zipkin200}. The inventory level is discrete and limited by the number of states $S$. The purchase cost, sale price, and holding cost are $2.49, 3.99$, and $0.03$ respectively. The demand is sampled from a normal distribution with a mean $S/4$ and a standard deviation of $S/6$. The initial state is $0$ (empty stock). \cref{tab:inventory} summarizes the computed guaranteed returns of different methods at $0.5$ and $0.95$ confidence levels. The guaranteed returns computed with Bayesian and weighted methods are significantly higher than other methods in this problem domain.

\begin{table*}
	\centering
	\begin{tabularx}{\textwidth}{r *5{>{\raggedleft\arraybackslash}X}}
		\toprule
		& \multicolumn{1}{c}{Confidence $\rightarrow$} & \multicolumn{2}{c}{0.5} & \multicolumn{2}{c}{0.95} \\
		\cmidrule(lr){3-4} \cmidrule(lr){5-6} 
		& \multicolumn{1}{c}{Methods $\downarrow$} & \multicolumn{1}{c}{Unweighted} & \multicolumn{1}{c}{Weighted} & 
		\multicolumn{1}{c}{Unweighted} & 
		\multicolumn{1}{c}{Weighted} \\
		\midrule
		\multirow{ 2}{*}{Bayesian} 
		& $L_1$ BCI & 24.17 & \textbf{26.45} & 23.87 & \textbf{26.41} \\
		& $L_{\infty}$ BCI & 23.94 & 26.35 & 23.63  & 26.24 \\ 
		\midrule
		\multirow{ 4}{*}{\centering Frequentist}
		& $L_1$ Hoeffding & 4.02 & 24.71 & 3.53 & 24.70 \\
		&$L_1$ Bernstein & 1.82 & 24.22 & 1.82 & 24.21 \\
		&$L_{\infty}$Hoeffding & 23.02 & \textbf{26.07} & 22.91 & \textbf{26.00} \\
		\bottomrule
	\end{tabularx}
	\caption{Cart-pole experiment. Guaranteed robust return for different confidence levels.} \label{tab:cartpole}
\end{table*}


\paragraph{Cart-Pole}
We evaluate our method on cart-pole, a standard RL benchmark problem \cite{sutton2018reinforcement,openaigym}. We collect samples of $100$ episodes from the true dynamics. We fit a linear model with that dataset to generate synthetic samples and aggregate nearby states on a resolution of 200 using K-nearest neighbor strategy. The results are summarized in \cref{tab:cartpole}. Again, in this case, all the Bayesian and weighted methods outperform other methods.

\section{Conclusion} \label{sec:conclusion}
In this paper, we proposed a new approach for optimizing the shape of the ambiguity sets that goes beyond the conventional $L_1$-constrained ambiguity sets studied in the literature. We showed that the optimal shape is problem dependent and is driven by the characteristics of the value function. We derived new sampling guarantees, and our experimental results show that the problem-dependent shapes of the ambiguity set can significantly improve solution quality.

\subsection{Acknowledgments}

This work was supported by the National Science Foundation under Grant Nos. IIS-1717368 and IIS-1815275.

\bibliography{reazul.bib}
\bibliographystyle{aaai}

\clearpage
\bigskip

\end{document}